%% file: main_arxiv.tex
\documentclass[10pt,twocolumn,letterpaper]{article}

\usepackage[pagenumbers]{cvpr} 

\input{preamble}

\definecolor{cvprblue}{rgb}{0.21,0.49,0.74}
\usepackage[pagebackref,breaklinks,colorlinks,allcolors=cvprblue]{hyperref}
\usepackage{media9} 

\def\paperID{} 
\def\confName{CVPR}
\def\confYear{2026}

\title{Vector Prism: Animating Vector Graphics by Stratifying Semantic Structure}

\author{Jooyeol Yun, \quad Jaegul Choo\\
KAIST AI\\
{\tt\small \{blizzard072, jchoo\}@kaist.ac.kr}
}

\begin{document}
\twocolumn[{
  \renewcommand\twocolumn[1][]{#1}
  \maketitle

  \begin{center}
    \captionsetup{labelformat=empty}
    \begin{minipage}[t]{.23\linewidth}
        \centering
        \animategraphics[controls=play,loop,width=\linewidth]{12}{animations/19-1-cool/frames/frame-}{000}{059}
        \captionof{figure}{I want the emoji to look to the left and right.}
    \end{minipage}\hspace{0.2cm}%
    \begin{minipage}[t]{.23\linewidth}
        \centering
        \animategraphics[controls=play,loop,width=\linewidth,poster=last]{12}{animations/6-1-rainbow/frames/frame-}{000}{059}
        \captionof{figure}{I want the elements smoothly pop up in a lively manner.}
    \end{minipage}\hspace{0.2cm}%
    \begin{minipage}[t]{.23\linewidth}
        \centering
        \animategraphics[controls=play,loop,width=\linewidth]{24}{animations/18-2-compass/frames/frame-}{000}{95}
        \captionof{figure}{I want the compass needle to quickly spin around once.}
    \end{minipage}\hspace{0.2cm}%
    \begin{minipage}[t]{.23\linewidth}
        \centering
        \animategraphics[controls=play,loop,width=\linewidth]{12}{animations/8-2-calculator/frames/frame-}{000}{035}
        \captionof{figure}{I want the buttons to bounce in one by one. }
    \end{minipage}
    \captionsetup{labelformat=default}
    \setcounter{figure}{0}    
    \captionof{figure}{Animations generated by Vector Prism. Please view them in Adobe Acrobat or the Firefox browser for the best experience. An HTML version is available in the \href{https://yeolj00.github.io/personal-projects/vectorprism}{project page}.}
    \label{fig:teaser}
  \end{center}
}]

\begin{abstract}
Scalable Vector Graphics (SVG) are central to modern web design, and the demand to animate them continues to grow as web environments become increasingly dynamic.
Yet automating the animation of vector graphics remains challenging for vision–language models (VLMs) despite recent progress in code generation and motion planning.
VLMs routinely mis-handle SVGs, since visually coherent parts are often fragmented into low-level shapes that offer little guidance of which elements should move together. 
In this paper, we introduce a framework that recovers the semantic structure required for reliable SVG animation and reveals the missing layer that current VLM systems overlook. 
This is achieved through a statistical aggregation of multiple weak part predictions, allowing the system to stably infer semantics from noisy predictions.
By reorganizing SVGs into semantic groups, our approach enables VLMs to produce animations with far greater coherence. 
Our experiments demonstrate substantial gains over existing approaches, suggesting that semantic recovery is the key step that unlocks robust SVG animation and supports more interpretable interactions between VLMs and vector graphics.
\end{abstract}

\section{Introduction}
Scalable Vector Graphics (SVG) has become increasingly central to modern web experiences, prized for its portability across devices and infinite scalability without quality loss. 
This popularity is driven by their vector-based design, which describes graphics through geometric primitives rather than pixels, resulting in compact and resolution-independent files. 
As modern web interfaces evolve toward dynamic and interactive experiences, the demand for expressive animation techniques has become essential, since SVG animations can deliver rich visual motion where videos would be prohibitively heavy for web delivery.

Recent advances in vision-language models (VLMs)~\cite{llava, gpt5, qwen3} offer a tempting possibility, which is generating animations simply by instructing a VLM given the SVG file. 
At first glance, this seems to be straightforward, since modern vision-language models can already plan animation sequences~\cite{wei2022chain} and generate code~\cite{swe, chen2021evaluating}. 
In practice, VLM-generated SVG animations rarely succeed, often resulting in visually broken animations. 
The problem lies not in the planning or coding capabilities, but in how SVGs are structured, as SVGs are optimized for rendering efficiency rather than semantic clarity. 
For example, as seen in \Cref{fig:intro-groups}, visually coherent elements (\eg, bunny ears and nose) are often fragmented or grouped by draw order, obscuring the higher-level semantics needed for animation.

In this paper, we address the overlooked step of restructuring SVGs so that vision-language models can reason about meaningful parts during animation. 
Our aim is to reveal an internal structure for SVGs that allows a model to reference semantic units and attach motion to correct semantic units. 
The native SVG hierarchy rarely provides this structure, which motivates a method that can reliably recover the semantics required for animation. 

We introduce Vector Prism, a frame work that performs this recovery by stratifying noisy visual cues into coherent semantic groups, much like a prism for vector graphics. 
Each SVG primitive (\ie, basic shapes) is rendered through several focused views (\eg, highlighting, isolation, zoom-in, outlining, and bounding boxes) and a VLM predicts its semantic meaning, producing a set of weak, tentative semantic labels. 
Instead of aggregating these predictions using simple majority voting, Vector Prism interprets these predictions through the lens of a statistical inference process~\cite{dawid1979maximum}.
Specifically, our method analyzes agreement patterns across weak labels and infers the underlying semantic signal with high stability. 
A Bayes decision rule then selects labels that minimize expected classification error and recover the most plausible true part structure.

These refined labels form the basis for the final stage, where Vector Prism restructures SVG primitives into coherent, animation-ready hierarchies. 
This restructuring bridges the gap between the visual semantics of the artwork and the syntactic organization of the SVG file, aligning the representation with how VLMs perceive and manipulate visual concepts. 
As a result, VLMs can animate graphics at the level of meaningful parts rather than low-level shapes, producing motions that are both visually stable and semantically consistent.

Our contributions are threefold:
\begin{itemize}
    \item We formalize the overlooked challenge that SVG files are structured for rendering efficiency rather than semantic clarity, making them ill-suited for animation. We introduce the problem of semantic SVG restructuring and propose a principled methodology for recovering animation-relevant part structure. 
    \item We propose Vector Prism, a statistical inference framework that transforms noisy view-dependent predictions from vision–language models into reliable semantic labels. By combining weak labels from multiple focused visualizations, our method infers robust underlying semantics.
    \item Our experiments demonstrate significant improvements over state-of-the-art methods in animation quality and instruction faithfulness, even outperforming commercial services such as Sora~2.
\end{itemize}

\begin{figure}
    \centering
    \includegraphics[width=\linewidth]{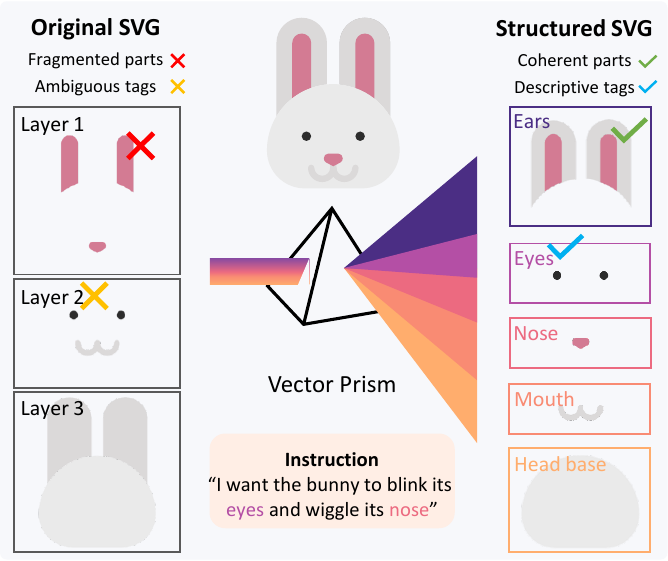}
    \vspace{-0.3cm}
    \caption{Unstructured SVG contains fragmented elements and unclear tags, while structured SVG organizes parts with descriptive tags, ensuring alignment between SVG syntax and user instructions. }
    \vspace{-0.1cm}
    \label{fig:intro-groups}
\end{figure}

\begin{figure*}
    \centering
    \includegraphics[width=\linewidth]{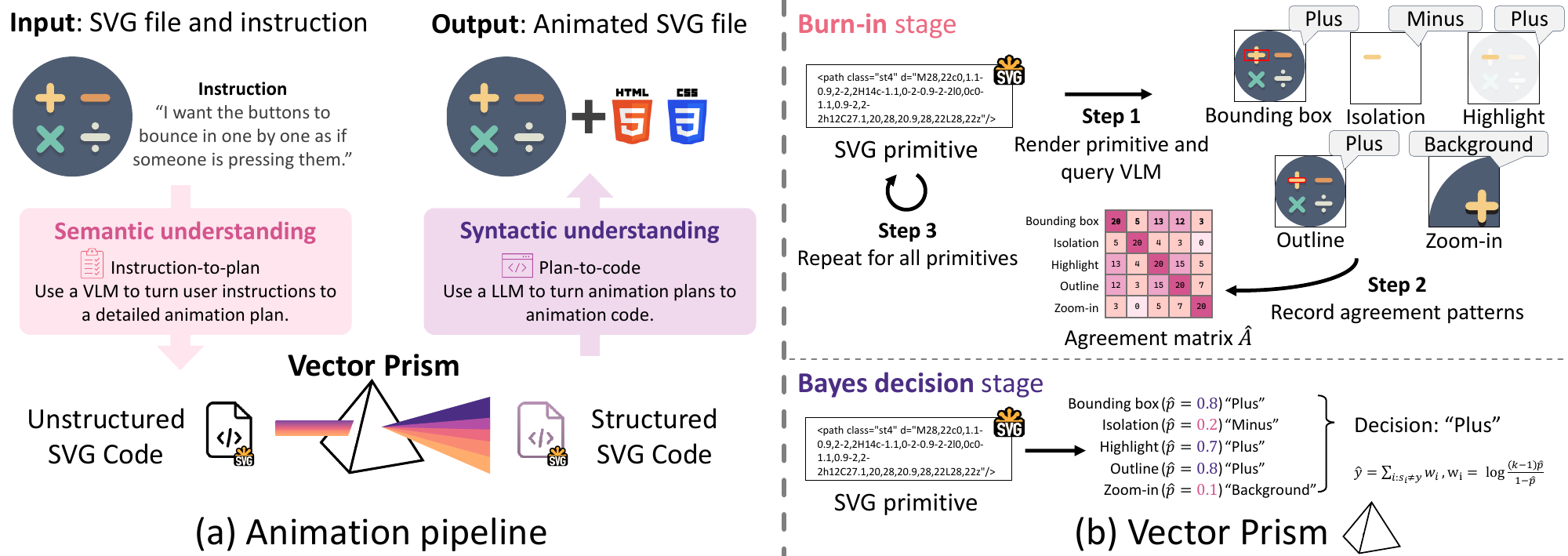}
    \label{fig:method-main}
    \caption{(a) Animation pipeline overview. We first create a detailed animation plan, then , then create the animation code for the structured SVG. (b) Vector Prism overview. We collect agreement patterns of response from different rendering methods and }
\end{figure*}

\section{Related Work}
\paragraph{Vector Graphics Animation}
One line of work generates or animates vector graphics by optimizing vector (or animation) parameters using gradients from pre-trained image/video diffusion priors~\cite{dynamictypo,aniclipart,stylecustomization,zhang2024text,wang2025layered}, typically via score distillation sampling (SDS)~\cite{sds,ldm,svd}. 
Since the SDS objective acts on rasterized renderings rather than vector structure, it encourages appearance preserving changes and resists large part rearrangements that animation often needs. 
Without explicit temporal regularization, the optimization often settles into short repetitive motions with visible jitter. 

Another active stream fine-tunes LLMs to directly produce vector graphics parameters or animation commands~\cite{starvector,omnisvg,wang2025internsvg}, enabled by large paired datasets of vector graphics and human instructions. 
Because LLMs carry little understanding for vector geometry and scene hierarchies~\cite{svgeditbench,zouvgbench}, performance scales primarily with data, often requiring millions of examples. 
Orthogonal to data scaling, we focus on recovering element-level semantics in the input SVGs, so that downstream LLMs/VLMs can robustly plan motions and generalize to diverse, in-the-wild graphics.

\paragraph{Semantic Understanding of Vector Graphics}
Raw symbolic representations of vector graphics (\eg, shape coordinates and translation matrices) are designed for rendering and programmatic manipulation rather than human reading or editing, which makes them inherently difficult for humans to directly inspect and understand~\cite{regroup,beyondpixels,cai2023leveraging,xing2025empowering}. 
This limitation has been highlighted as the research community seeks to teach LLMs, which often rely on perceptual cues similar to humans, to understand and edit vectorized formats~\cite{svgeditbench,lin2025vcode,qiu2024can,zouvgbench}.

Although VLMs tend to understand rasterized renderings of simple and well-separated vector graphics~\cite{beyondpixels}, we find that they quickly fail to understand individual parts of complex real-world cases. 
In this paper, we take a significant step in vector graphics understanding by aiming not only to target complex real-world SVGs but to identify and label individual SVG primitives, which is exactly the capability required for animation. 
To do so, we present a statistical inference framework that makes unreliable and noisy LLM outputs into robust decisions, enabling animation possible even without finetuning VLMs.

\section{Method}
\subsection{The overview}
As illustrated in \Cref{fig:method-main}, the pipeline begins with animation planning (\ref{sec:method-plan}), where a vision–language model (VLM) interprets the visual content and generates a detailed scheme of how each semantic components should be animated.
It then proceeds to semantic wrangling (\ref{sec:method-vp}), where the SVG is restructured into a semantically meaningful and animatable form through a statistical inference, and finally to animation generation (\ref{sec:method-animate}), which produces executable animation code. 

The planning stage provides semantic understanding of the scene, while the animation stage operates directly on the SVG code.
Our restructuring stage bridges this gap by injecting semantic meaning into the SVG, enriching its structure with interpretable tags that connect visual reasoning to code-level representation.
The core contribution of our approach lies in this stage, where we introduce a statistical inference framework that makes reliable semantic inferences from inherently noisy model predictions.

\subsection{Animation planning}
\label{sec:method-plan}
The planning stage uses a VLM to reason about the scene at a semantic level. 
The SVG is first rendered into a raster image so it can be understood by the VLM, which offers strong visual signals compared to the original SVG code representations.  
The VLM is then instructed to produce high-level animation plans given the rendered image and the user's animation description, identifying which semantic components should move and how they relate to one another.
For example, when prompted to \emph{``make the sun rise,''} the VLM identifies the circular yellow region as the sun and the blue background as the sky, proposing that the sun should move upward while the sky gradually brightens.

Since VLMs lack an understanding of the symbolic structures (\ie, SVG syntax), they have no way to directly implement those plans into the SVG's syntactic hierarchy. 
Bridging this semantic–syntactic divide is precisely the role of the restructuring stage.

\subsection{Vector Prism}
\label{sec:method-vp}

\paragraph{Problem setup and notations}
Given a SVG file, let  $\mathcal{X}$ be the set of all the primitives, which are basic shapes such as \texttt{<path>}, \texttt{<rect>}, \texttt{<circle>}, \texttt{<ellipse>}, \texttt{<line>}, \texttt{<polyline>}, and \texttt{<polygon>}.
Every primitive should fall into the one of the semantic categories $\mathcal{Y}=\{1,\dots,k\}$ fixed from the planning stage. 
For each primitive $\vx\in\mathcal{X}$ there is an unknown true label $y(\vx)\in\mathcal{Y}$ that we want to predict.

For a SVG primitive to be visually interpreted by a VLM, it first needs to be rendered into a raster image. 
Deciding how to render $\vx$ is non-trivial, and thus we use $M$ different rendering methods indexed by $i\in\{1,\dots,M\}$. 
This provides complementary views of the target primitive, which helps us safely collect different weak labels of the same primitive. 
Examples include highlight on the original canvas, a tight bounding-box overlay, a zoomed crop, and isolation on a blank background.
When we use method $i$ to render a primitive $\vx$, the VLM returns a label $s_i(\vx)\in\mathcal{Y}$.

We assume a Dawid-Skene model~\cite{dawid1979maximum} for each rendering method, 
\[
\Pr[s_i=\ell]=
\begin{cases}
p_i, & \ell=y,\\
\frac{1-p_i}{k-1}, & \ell\neq y.
\end{cases}
\]
where a rendering method $i$ has as accuracy $p_i$ and fails uniformly over the other $k-1$ labels. 
We will recover the reliability $p_i$ of each strategy. 

\paragraph{From pairwise agreement to reliability}
Under the model above, VLM responses from two different renderings $i$ and $j$ would agree either because both are correct or because both pick the same wrong label.
Thus, the probability of agreement is
\begin{equation}
\mA_{ij}=\Pr[s_i=s_j]=p_i p_j+\frac{(1-p_i)(1-p_j)}{k-1}.
\label{eq:agree}
\end{equation}
Since two random guesses could still agree by chance with probability $1/k$, we write $\delta_i=p_i-\tfrac{1}{k}$ and
\begin{equation}
\mA_{ij}=\frac{1}{k}+\frac{k}{k-1}\,\delta_i\delta_j \quad (i\neq j).
\label{eq:agree-decomp}
\end{equation}
to separate chance from skill. 
Subtracting the chance term gives a centered agreement matrix $\mB$ with
$\mB_{ij}=\mA_{ij}-\tfrac{1}{k}$ for $i\neq j$ and $\mB_{ii}=0$.
Matrix $\mB$ is rank one on the off-diagonals
\begin{equation}
\E[\mB]=\frac{k}{k-1}\,\bm{\delta}\bm{\delta}^\top,
\label{eq:rank1}
\end{equation}
which is the outer product of $\delta$. 
Let $\lambda$ and $\vv$ be the top eigenvalue and eigenvector of $\mB$, then 
\[
\bm\delta=\sqrt{\frac{\lambda(k-1)}{k}}\,\vv,
\qquad
p_i=\frac{1}{k}+\delta_i,
\] 
with the sign of $\vv$ chosen so that $\sum_i \hat\delta_i\ge 0$. 
In this way, given the agreement matrix $\mA$, we can recover the initially unknown reliability of each VLM response $i$. 

The agreement matrix $\mA$ can be empirically estimated by a burn-in pass, traversing the SVG primitives and collecting the agreement patterns
\[
\hat \mA_{ij}=\frac{1}{|\mathcal{X}|}\sum_{\vx\in\mathcal{X}}\mathbf{1}[s_i(\vx)=s_j(\vx)].
\]
Following \Cref{eq:agree-decomp} and \Cref{eq:rank1}, we can obtain $\hat{\bm\delta}$ and consequently a reliability $\hat{p}_i$ for each rendering method. 

\paragraph{From reliability to semantic labels}
With reliabilities $\hat p_i$ in hand, we score each candidate label $y\in\mathcal{Y}$ for a given element using Bayes' decision rule with a uniform prior
\[
\log P(y\mid s)=\mathrm{const}
+\sum_{i:\, s_i=y}\log \hat p_i
+\sum_{i:\, s_i\ne y}\log\frac{1-\hat p_i}{k-1}.
\]
This is equivalent to a weighted vote with
\[
w_i=\log\frac{(k-1)\hat p_i}{1-\hat p_i}, \qquad
\hat y=\argmax_y \sum_{i:\, s_i=y} w_i.
\]
When all VLMs are equally reliable, all $w_i$ are equal and the decision rule reduces to majority voting. 
A probability bound comparing this rule to majority voting and showing a strict advantage whenever VLM reliabilities differ, is provided in the Appendix. 

\paragraph{From semantic labels to a new structure}
\label{sec:method-restructure}
Once reliable semantic labels are available, restructuring the SVG becomes a straightforward step that turns meaning into organization without changing appearance.
Although SVGs are usually grouped for rendering efficiency, not semantics, this step only needs to reorganize existing elements rather than reinterpret them.
For example, shapes that share similar transformations may be grouped together even if they represent different objects, causing unrelated parts to move together.
With correct labels, this can be easily fixed.

Our restructuring algorithm attaches each label as a \texttt{class} attribute and flattens the hierarchy so that all visual properties are applied directly to each primitive, preserving appearance.
Primitives are then regrouped by label while maintaining the original paint order.
Overlaps between different labels are checked to prevent rendering changes.
The resulting SVG looks identical but is organized into meaningful parts ready for animation.
Full pseudocode are provided in the Appendix and the code will be released upon acceptance.

\subsection{Animation generation}
\label{sec:method-animate}
The LLM is instructed to animate the restructured SVG file according to the animation plan using CSS. 
While the earlier pipeline steps do not restrict generating animations to the CSS markup type, CSS was chosen for its simplicity, and our method has the capability to extend to complex animations using JavaScript or specialized libraries. 

Animation code can become lengthy, often exceeding the token generation limits of many models.
To address this constraint, we adopt an iterative generation strategy~\cite{alphacode,codegen}, where CSS animations are generated separately for each semantic category, with previously completed animations retained in the context for subsequent generations. 
To prevent conflicting animations, we enforce strict animation rules that ensure mutual exclusivity between generated effects. 
Complete prompts are provided in the Appendix.

\section{Experiments}
\subsection{Dataset}
Our test dataset consists of 114 carefully curated animation instructions and SVG pairs, designed to test a variety of SVG animation techniques. 
The instruction set covers a broad range of animation tasks, from simple movements to complex actions such as 3D rotations and synchronized transitions. 
The SVG files were sourced from SVGRepo, ensuring a diverse collection of objects and scenes, including animals, logos, buildings, and natural elements like fire, clouds, and water. 
The goal of this dataset is to evaluate the performance of SVG animation tools and systems by providing clear, detailed animation instructions that simulate real-world use cases in web environments. 
The animation categories and their performance are discussed in detail in the appendix. 

\subsection{Baselines}
\paragraph{AniClipart}
AniClipart~\cite{aniclipart} represents the optimization-based animation methods, which optimizes animation parameters such as keypoint movements, using the Score Distillation Sampling loss~\cite{sds}. 
While AniClipart does not output standard animation formats, it defines B\'ezier curves for keypoints within SVG files, enabling direct vector graphics animation.

\paragraph{GPT-5}
GPT-5 is reported to have one of the best understanding of symbolic representation among LLMs~\cite{gpt5}. 
However, we observe that naive prompting of LLMs to generate animation code rarely produces meaningful motion. 
Therefore, we augment GPT-5 with the same high-level planning and animation generation pipeline employed in our framework to ensure fair comparison. 
In this configuration, GPT-5 generates CSS animations in vector format.

\paragraph{Video generation models}
We include two video generation models, the open-sourced Wan2.2 14B model~\cite{wan2025} and OpenAI's Sora2 service~\cite{sora2}. 
Although these models produce rasterized video output (\texttt{.mp4}) and cannot generate the desired vector files, we include them to cover a wide scope of animation generation technique, especially as these models demonstrate high performances in instruction following and video quality. 

\begin{table}[]
\resizebox{\linewidth}{!}{
\begin{tabular}{@{}ccccc@{}}
\toprule
                            & CLIP-T2V & GPT-T2V & DOVER & Vector \\ \midrule
AniClipart~\cite{aniclipart}& 15.66    & 23.96   & 3.35 & \cmark \\
GPT-5~\cite{gpt5}           & 20.67    & 40.92   & 4.92 & \cmark \\
\rowcolor{gray!20} 
Wan 2.2~\cite{wan2025}      & 21.14    & 65.21   & 3.72 & \xmark \\ 
\rowcolor{gray!20} 
Sora 2~\cite{sora2}         & 20.29    & 69.08   & 4.19 & \xmark \\ \midrule
Ours  & \textbf{21.55} & \textbf{76.14} & \textbf{4.97} & \cmark \\ \bottomrule
\end{tabular}
}
\vspace{-0.2cm}
\caption{Animation quality and instruction-following scores across different methods. The checkmark indicating whether each method generates vector-based animations.}
\vspace{-0.5cm}
\label{tab:exp-quanti}
\end{table}

\subsection{Implementation Details}
We use GPT-5-nano, which is $25\times$ more cost-efficient than GPT-5, as the underlying vision–language model for planning and semantic labeling, while GPT-5 is used for animation generation. 
Our semantic labeling stage is statistically robust to noise and operates with minimal computational overhead, enabling lightweight models to perform reliably without sacrificing accuracy.
All SVG primitives are rendered at $512\times512$ resolution when given as a VLM input for analysis. 

We do not share the agreement matrix across SVGs, since we find that the reliability of each rendering method can vary depending on the visual complexity and structure of the SVG. 
During the burn-in stage, where agreement patterns are collected, a single full pass over all primitives within each SVG provides a good balance between estimation stability and computational efficiency. 

\subsection{Quantitative Evaluation}
We evaluate the generated animations using two instruction-following metrics and one perceptual quality metric.
Following InternSVG~\cite{wang2025internsvg}, we measure the correspondence between animation instructions and rendered videos using a video-pretrained CLIP model~\cite{viclip,clip}, referred to as CLIP-T2V. 
To complement this, we introduce the GPT-T2V score, where GPT grades each video based on how accurately its motion follows the given instruction. 
This follows the growing use of LLM-based evaluators for instruction following and multimodal reasoning~\cite{zheng2023judging}. 
Finally, we assess perceptual quality with DOVER~\cite{wu2023dover}, an off-the-shelf video quality assessment model that captures both technical fidelity and visual aesthetics. 
Also, a trade-off between perceptual quality and instruction following can easily occur, as limiting motion often leads to higher visual quality, whereas enforcing movement to meet the instruction can reduce perceptual fidelity. 

As shown in \Cref{tab:exp-quanti}, our method achieves the best scores across all metrics, demonstrating clear advantages in both motion realism and instruction faithfulness. 
This improvement comes from the ability to expose meaningful parts of the SVG prior to animation, allowing the model to attach coherent motions to relevant semantic parts. 
It is also important to note that vector-based animations typically struggle with instruction-following compared to video generation models, as video models are heavily trained on video-text pairs. 
However, this limitation is not inherent to the vector-based format, but rather stems from the lack of semantic understanding of vectors. 
Our method overcomes this and outperforms video models as well, without training on large-scale video video-text datasets.

\begin{figure*}[ht]
    \centering
    \vspace{-0.5cm}
    \begin{tabular}{ccccc}
        \multicolumn{5}{c}{\textbf{I want an opening animation for the SVG, starting from the bottom and moving up to the top.}} \\
        \multicolumn{5}{c}{ } \\
        \begin{minipage}{0.18\textwidth}
            \centering
            \animategraphics[controls=play,loop,width=0.9\linewidth]{12}{videos/57-2-aniclip/frames/frame-}{000}{035}
            \caption*{AniClipart}
        \end{minipage} &
        \begin{minipage}{0.18\textwidth}
            \centering
            \animategraphics[controls=play,loop,width=0.9\linewidth]{12}{videos/57-2-gpt/frames/frame-}{000}{59}
            \caption*{GPT-5}
        \end{minipage} &
        \begin{minipage}{0.18\textwidth}
            \centering
            \animategraphics[controls=play,loop,width=0.9\linewidth]{12}{videos/57-2-wan/frames/frame-}{000}{60}
            \caption*{Wan~2.2}
        \end{minipage} &
        \begin{minipage}{0.18\textwidth}
            \centering
            \animategraphics[controls=play,loop,width=0.9\linewidth]{12}{videos/57-2-sora/frames/frame-}{000}{047}
            \caption*{Sora~2}
        \end{minipage} &
        \begin{minipage}{0.18\textwidth}
            \centering
            \animategraphics[controls=play,loop,width=0.9\linewidth,poster=last]{12}{animations/57-2-windmill/frames/frame-}{000}{059}
            \caption*{Ours}
        \end{minipage} \\ 
        \hline

        \multicolumn{5}{c}{ } \\
        \multicolumn{5}{c}{\textbf{I want the lightning bolt to glow softly and the raindrops to fade in and out gently.}} \\
        \multicolumn{5}{c}{ } \\
        \begin{minipage}{0.18\textwidth}
            \centering
            \animategraphics[controls=play,loop,width=0.9\linewidth]{12}{videos/16-2-aniclip/frames/frame-}{000}{035}
            \caption*{AniClipart}
        \end{minipage} &
        \begin{minipage}{0.18\textwidth}
            \centering
            \animategraphics[controls=play,loop,width=0.9\linewidth]{12}{videos/16-2-gpt/frames/frame-}{000}{59}
            \caption*{GPT-5}
        \end{minipage} &
        \begin{minipage}{0.18\textwidth}
            \centering
            \animategraphics[controls=play,loop,width=0.9\linewidth]{12}{videos/16-2-wan/frames/frame-}{000}{60}
            \caption*{Wan~2.2}
        \end{minipage} &
        \begin{minipage}{0.18\textwidth}
            \centering
            \animategraphics[controls=play,loop,width=0.9\linewidth]{12}{videos/16-2-sora/frames/frame-}{000}{047}
            \caption*{Sora~2}
        \end{minipage} &
        \begin{minipage}{0.18\textwidth}
            \centering
            \animategraphics[controls=play,loop,width=0.9\linewidth]{12}{animations/16-2-lightning/frames/frame-}{000}{059}
            \caption*{Ours}
        \end{minipage} \\
        \hline

        \multicolumn{5}{c}{ } \\
        \multicolumn{5}{c}{\textbf{I want the stars and planets first to emerge gently and then the rings to appear in a stroke effect.}} \\
        \multicolumn{5}{c}{ } \\
        \begin{minipage}{0.18\textwidth}
            \centering
            \includegraphics[width=0.9\linewidth]{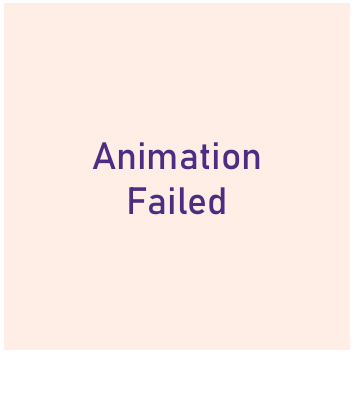}
            \caption*{AniClipart}
        \end{minipage} &
        \begin{minipage}{0.18\textwidth}
            \centering
            \animategraphics[controls=play,loop,width=0.9\linewidth]{12}{videos/32-1-gpt/frames/frame-}{000}{59}
            \caption*{GPT-5}
        \end{minipage} &
        \begin{minipage}{0.18\textwidth}
            \centering
            \animategraphics[controls=play,loop,width=0.9\linewidth]{12}{videos/32-1-wan/frames/frame-}{000}{60}
            \caption*{Wan~2.2}
        \end{minipage} &
        \begin{minipage}{0.18\textwidth}
            \centering
            \animategraphics[controls=play,loop,width=0.9\linewidth]{12}{videos/32-1-sora/frames/frame-}{000}{047}
            \caption*{Sora~2}
        \end{minipage} &
        \begin{minipage}{0.18\textwidth}
            \centering
            \animategraphics[controls=play,loop,width=0.9\linewidth,poster=30]{12}{animations/32-1-galaxy/frames/frame-}{000}{059}
            \caption*{Ours}
        \end{minipage} \\
        \hline

        \multicolumn{5}{c}{ } \\
        \multicolumn{5}{c}{\textbf{I want the hexagon to appear first, and then the X sign to enter by spinning in.}} \\
        \multicolumn{5}{c}{ } \\
        \begin{minipage}{0.18\textwidth}
            \centering
            \includegraphics[width=0.9\linewidth]{figures/exp_noimage.pdf}
            \caption*{AniClipart}
        \end{minipage} &
        \begin{minipage}{0.18\textwidth}
            \centering
            \animategraphics[controls=play,loop,width=0.9\linewidth]{12}{videos/12-1-gpt/frames/frame-}{000}{59}
            \caption*{GPT-5}
        \end{minipage} &
        \begin{minipage}{0.18\textwidth}
            \centering
            \animategraphics[controls=play,loop,width=0.9\linewidth]{12}{videos/12-1-wan/frames/frame-}{000}{60}
            \caption*{Wan~2.2}
        \end{minipage} &
        \begin{minipage}{0.18\textwidth}
            \centering
            \animategraphics[controls=play,loop,width=0.9\linewidth]{12}{videos/12-1-sora/frames/frame-}{000}{047}
            \caption*{Sora~2}
        \end{minipage} &
        \begin{minipage}{0.18\textwidth}
            \centering
            \animategraphics[controls=play,loop,width=0.9\linewidth,poster=last]{12}{animations/12-1-cancel/frames/frame-}{000}{59}
            \caption*{Ours}
        \end{minipage} \\
    \end{tabular}
    \caption{Animations generated by each method. Please use Adobe Acrobat, the Firefox browser, or the PDF.~js extension on Chromium browsers for the best experience~\cite{theanimatepackage}. An HTML version is available in the project page.}
    \label{fig:exp-quali}
\end{figure*}

\begin{figure}
    \centering
    \includegraphics[width=\linewidth]{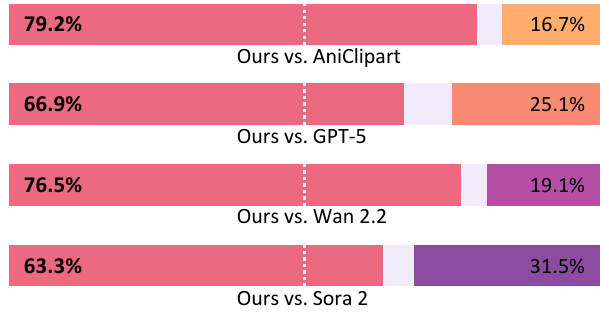}
    \caption{Human preference results comparing our method with baseline approaches. Pink segments represent preferences towards our method, and orange or purple segments represent the competing baseline.}
    \label{fig:exp-userstudy}
    \vspace{-0.5cm}
\end{figure}

\subsection{Qualitative Analysis}
AniClipart and GPT-5 often fail to produce meaningful motion since they lack explicit semantic understanding. 
These approaches interpret semantics implicitly, AniClipart through the diffusion prior and GPT-5 through internal representations, without explicit part labels or hierarchy. 
As a result, they tend to produce uniform motion across entire figures, leading to swaying or barely moving animations.

Video generation models, Wan~2.2 and Sora~2, generate richer motion than the above methods but often collapse into static frames or distorted scenes when given dynamic, animation-focused instructions such as ``An opening scene of the SVG.'' 
Note that these are rasterized videos rather than vector graphics, which makes them unsuitable for web-based animation tasks where lightweight rendering is essential. 
In contrast, our method translates instructions into motion entirely through the language domain, avoiding the limitations of multimodal training and dataset dependence. 

We showcase examples of the generated animations in \Cref{fig:exp-quali}, where the differences discussed above are clearly visible. 
Additional qualitative results and extended comparisons are provided in the project page for further reference.

\subsection{User Study}
To complement the quantitative evaluation, we conducted a user study to assess how well each generated animation aligns with the given instructions from a human perspective.
A total of 760 pairwise comparisons were collected from 19 participants.
In each trial, participants were shown two videos generated from the same instruction, each produced by a different method, and asked to select the one that better followed the instruction.
The aggregated preferences are summarized in \Cref{fig:exp-userstudy}, showing consistent favorability toward our method even when compared against state-of-the-art video generation models such as Sora~2 and Wan~2.2. 
We report the alignment between the user study outcomes and the GPT-T2V metrics in the Appendix. 

\begin{figure}
    \centering
    \includegraphics[width=\linewidth]{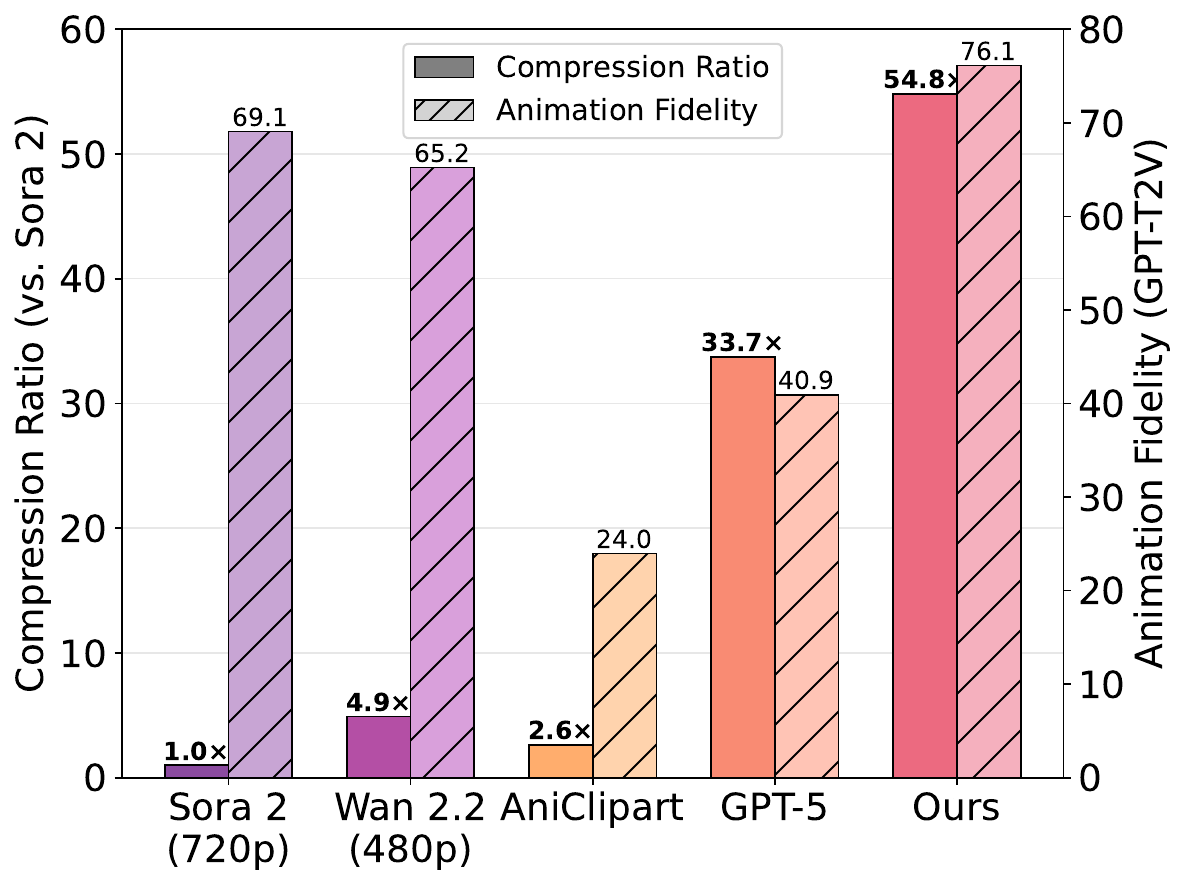}
    \vspace{-0.2cm}
    \caption{Dual-axis bar chart comparing compression ratio (left y-axis) and animation fidelity (right y-axis). Compression ratios are depicted by solid bars, and Animation Fidelity is shown with hatched bars. }
    \label{fig:exp-encoding}
\end{figure}

\section{Analysis}
\subsection{Encoding efficiency of vector-based animations}
We demonstrate the effectiveness of vector-based animations by comparing the compression ratio compared to Sora~2 and the animation fidelity in \Cref{fig:exp-encoding}. 
Typically, as the raster video resolution increases for quality (\eg, from 480p to 720p), the file size increases and the video less compressed. 
In contrast, the SVG animations generated by our approach describe motion through compact, symbolic CSS keyframes applied to geometric primitives. 
The resulting file size is primarily dependent on the complexity of the SVG structure (number of primitives) and the length of the animation code, not the output resolution or frame rate. 

This leads to a significant improvement in encoding efficiency compared to video models like Wan 2.2 and Sora 2, which generate every pixel of the animation, even when a vector representation is possible. 
Sora~2, for instance, results in an average file size that is $\times 54$ larger than those produced by our approach, with this gap widening as video resolution and duration increase. 
This makes our approach particularly well-suited for modern web environments, where lightweight assets are essential for fast loading times, responsive UI/UX, and reduced data consumption across networks.

\subsection{Stability compared to majority voting}
Evaluating the quality of semantic groupings in SVGs is challenging without ground truth labels, yet crucial for understanding whether our statistical inference produces coherent clusters. 
We treat each semantic group as a cluster and measure clustering quality using the Davies-Bouldin index (DBI)~\cite{dbi}, a metric that quantifies the ratio of within-cluster scatter to between-cluster separation.
We compute distances in the feature space of DINO v3~\cite{simeoni2025dinov3}, which provides semantically meaningful visual embeddings.

SVG files with their original, rendering-oriented groupings yield an average DBI of 33.8, reflecting the semantic incoherence of primitives grouped solely for drawing efficiency.
Majority voting with the same multi-view rendering techniques improves this to 12.6, demonstrating that aggregating multiple views helps, but still produces noisy groupings.
In contrast, Vector Prism achieves a DBI of \textbf{0.82}, indicating near-perfect semantic clustering.

The key advantage of our approach over majority voting is illustrated in \Cref{fig:analysis-dbi}.
When one rendering method produces unreliable predictions, correct only by chance, majority voting treats it equally with other reliable methods. 
This equal weighting allows the weakest reliable responses to occasionally flip the predicted label for certain primitives, creating inconsistent groupings that fragment semantically coherent parts.
Since animation quality depends on \emph{all} primitives being correctly grouped, even a small fraction of mislabeled elements can break the visual logic of motion.
By estimating reliability scores $\hat{p}_i$ for each rendering method, Vector Prism consistently downweights noisy VLM responses throughout the entire labeling process, ensuring stability across the full set of primitives.

\begin{figure}
    \centering
    \includegraphics[width=\linewidth]{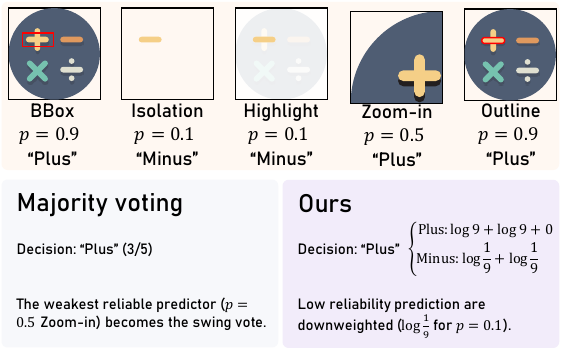}
    \caption{Example case of when Bayes decision rule can consistently make robust decisions even with noisy signals.}
    \label{fig:analysis-dbi}
\end{figure}

\subsection{Failure Cases}
Even with semantic groupings and a well-structured hierarchy, our method operates at the level of SVG primitives defined in the original file.
We treat primitives as atomic units and do not subdivide or decompose them further, which limits animation flexibility when the input SVG lacks granularity.
For example, as shown in \Cref{fig:analysis-failure}, the lightning shape is written as a single large \texttt{<path>} element, while the instruction requires this to ``shatter into pieces.''
The method fails to animate this part of the instruction, as the pieces themselves do not exist as independent primitives. 

This limitation could be addressed if users can refine their SVG files using vectorization tools such as VTracer~\cite{vtracer} or recent image-to-SVG models~\cite{starvector,omnisvg}, which generate SVGs with controllable levels of detail.
Alternatively, future work could explore automatic primitive subdivision strategies that identify and split overly coarse elements based on the animation requirements.

\begin{figure}
    \centering
    \adjustbox{valign=c}{%
        \includegraphics[width=0.68\linewidth]{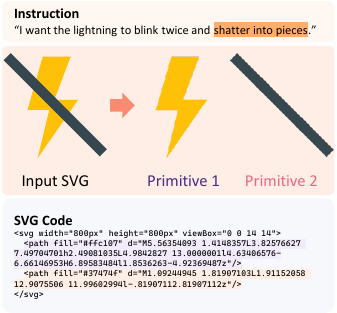}%
    }%
    \hfill
    \adjustbox{valign=c}{%
        \animategraphics[controls=play,loop,width=0.27\linewidth]{12}%
          {animations/10-2-flash/frames/frame-}{000}{35}%
    }%
    \caption{Failure case. Since the lightning bolt is defined as a single atomic \texttt{<path>} primitive (left), our approach cannot execute the operation beyond the input SVG's granularity (right).}
    \vspace{-0.4cm}
    \label{fig:analysis-failure}
\end{figure}

\section{Conclusion}
In this paper, we introduced Vector Prism, a novel framework designed to overcome the critical semantic-syntactic gap that prevents modern vision-language models (VLMs) from successfully animating Scalable Vector Graphics (SVGs). 
Our core insight is that by enriching the native SVG structure with coherent semantic anchors, VLMs can reason about meaningful parts and reliably generate targeted motion.
The foundation of our approach is a multi-view statistical inference mechanism utilizing the Dawid-Skene model, which effectively transforms noisy, weak predictions from a VLM into robust, high-confidence semantic labels, eliminating the need for extensive, domain-specific VLM fine-tuning.
Through rigorous quantitative and qualitative evaluations, we demonstrated that our method achieves unmatched improvements in animation quality and instruction fidelity, surpassing both existing vector animation techniques and state-of-the-art raster video generation models. 

We believe that bridging the semantic-syntactic gap is a vital, generalizable step for unlocking the full potential of VLMs across various symbolic domains. 
Whether for vector graphics or for 3D assets and scenes, methods that align human semantic intent with machine-readable structure will significantly broaden the capabilities of language models, transforming them from passive code generators into robust, context-aware animation and design agents.

{
    \small
    \bibliographystyle{ieeenat_fullname}
    \bibliography{main}
}

\input{appendix}

\end{document}

%% file: preamble.tex
\usepackage{graphicx}
\usepackage{amsmath}
\usepackage{amsthm} 
\usepackage{amssymb}
\usepackage{booktabs}

\usepackage{color}
\usepackage{varwidth}
\usepackage{xcolor}
\usepackage[table,xcdraw]{xcolor} 
\usepackage{soul} 

\usepackage{verbatim}
\usepackage{multirow}
\usepackage{adjustbox}
\usepackage{array}
\usepackage{bbm}
\usepackage{mathtools}
\usepackage{pifont}
\newcommand{\cmark}{\textcolor{green!60!black}{\ding{51}}}
\newcommand{\xmark}{\textcolor{red!70!black}{\ding{55}}}
\usepackage{scalerel}
\usepackage{listings}
\newcolumntype{x}[1]{>{\centering\let\newline\\\arraybackslash\hspace{0pt}}p{#1}}

\usepackage[symbol]{footmisc}

\usepackage{dirtytalk}
\usepackage{animate}
\input{math_commands}

\newtheorem{theorem}{Theorem}[section]

\newtheorem{lemma}[theorem]{Lemma}
\newtheorem{proposition}{Proposition}
\theoremstyle{definition}

\theoremstyle{remark}
\newtheorem{remark}[theorem]{Remark}
 
\usepackage{algorithm} 
\usepackage{algpseudocode}
\usepackage{caption} 
\usepackage{subcaption}

\usepackage{tocloft}
\usepackage{setspace}

\usepackage{lipsum}



%% file: math_commands.tex
\usepackage{amsmath,amsfonts,bm}

\def\eqref#1{equation~\ref{#1}}

\def\1{\bm{1}}

\def\vv{{\bm{v}}}

\def\vx{{\bm{x}}}

\def\mA{{\bm{A}}}
\def\mB{{\bm{B}}}

\DeclareMathAlphabet{\mathsfit}{\encodingdefault}{\sfdefault}{m}{sl}
\SetMathAlphabet{\mathsfit}{bold}{\encodingdefault}{\sfdefault}{bx}{n}

\newcommand{\E}{\mathbb{E}}

\newcommand{\R}{\mathbb{R}}

\DeclareMathOperator*{\argmax}{arg\,max}

%% file: appendix.tex
\clearpage
\setcounter{section}{0}
\setcounter{page}{1}
\renewcommand*{\thesection}{\Alph{section}}
\onecolumn
{\newpage\centering\Large\textbf{\thetitle}\\
\vspace{0.5em}Supplementary Material \\
\vspace{1.0em}
}

\begin{center}
{\Large \textbf{Contents}}\\[1.2em]
\end{center}

\renewcommand{\cftsecfont}{\Large}
\renewcommand{\cftsecpagefont}{\Large}
\renewcommand{\cftsubsecfont}{\large}
\renewcommand{\cftsubsecpagefont}{\large}

\renewcommand{\cftsecleader}{\cftdotfill{\cftdotsep}}
\renewcommand{\cftsubsecleader}{\cftdotfill{\cftdotsep}}

\renewcommand{\cftbeforesecskip}{3cm}
\renewcommand{\cftbeforesubsecskip}{2cm}

\setlength{\cftsecindent}{0pt}
\setlength{\cftsecnumwidth}{3.2em}

\setlength{\cftsubsecindent}{3.2em}
\setlength{\cftsubsecnumwidth}{4em}

\begin{spacing}{1.4}
\begin{flushleft}

\noindent\Large\textbf{A}\quad \hyperref[anchor-extended]{Extended Related Work} \dotfill \pageref{anchor-extended}
\vspace{+0.8cm}

\noindent\Large\textbf{B}\quad \hyperref[anchor-confidence]{Confidence Bounds for Reliability Estimation} \dotfill \pageref{anchor-confidence}
\vspace{+0.1cm}

\newcommand{\tocsub}[3]{%
  \noindent\hspace*{3.2em}\mbox{\large\textbf{#1}\quad #2}%
  \dotfill #3
}

\tocsub{B.1}{\hyperref[anchor-agree]{Decomposing Agreements and Rank-One Structure}}{\pageref{anchor-agree}}
\vspace{+0.1cm}

\tocsub{B.2}{\hyperref[anchor-spectral]{Reliability Estimation via Eigenvector Analysis}}{\pageref{anchor-spectral}}
\vspace{+0.1cm}

\tocsub{B.3}{\hyperref[anchor-bounds]{Bayes Rule vs. Majority Voting: Error Bounds}}{\pageref{anchor-bounds}}
\vspace{+0.8cm}

\noindent\Large\textbf{C}\quad \hyperref[anchor-gpt]{GPT-Human Alignment on Video–Text Alignment} \dotfill \pageref{anchor-gpt}
\vspace{+0.8cm}

\noindent\Large\textbf{D}\quad \hyperref[anchor-dataset]{Dataset Composition and Coverage} \dotfill \pageref{anchor-dataset}
\vspace{+0.8cm}

\noindent\Large\textbf{E}\quad \hyperref[anchor-prompts]{VLM Prompts for Planning and Animation Generation} \dotfill \pageref{anchor-prompts}
\vspace{+0.8cm}

\noindent\Large\textbf{F}\quad \hyperref[anchor-restruct]{Restructuring SVG Files with Semantic Labels} \dotfill \pageref{anchor-restruct}
\vspace{+6cm}

\end{flushleft}
\end{spacing}

\hrule
\vspace{+0.2cm}


\clearpage

\section{Extended Related Work}
\phantomsection
\label{anchor-extended}

\subsection{Vector graphics generation and decomposition}
Beyond animation, the broader field of vector graphics generation has mainly focused on vectorizing raster images. 
DeepSVG~\cite{deepsvg} introduced a hierarchical generative network that jointly models both the structure and appearance of vector graphics, enabling controllable generation through learned latent representations.
More recently, LayerPeeler~\cite{layerpeeler} proposed an autoregressive approach to decompose raster images into layered vector representations, demonstrating that careful layer-wise decomposition can produce more interpretable and editable vector graphics.
Yuan~\etal~\cite{yuan2025rdtf} extended these ideas to animated stickers, introducing a resource-efficient dual-mask training framework that generates multi-frame animations while maintaining computational efficiency.
While our method assumes a user-provided SVG as input, it can be seamlessly combined with image-to-SVG or text-to-SVG models to synthesize SVG animations directly from images or text, eliminating the need for manually authored SVGs.

\subsection{Domain-specific SVG understanding}
The challenge of understanding structured visual representations extends to specialized domains such as data visualization, such as graphs and charts. 
Chen~\etal~\cite{chen2023mystique} developed Mystique, a system for deconstructing SVG charts to enable layout reuse, demonstrating that reverse-engineering the semantic structure of charts requires domain-specific parsing strategies.
Building on this, VisAnatomy~\cite{chen2024visanatomy} provided a large-scale corpus of SVG charts with fine-grained semantic labels, establishing benchmarks for chart understanding tasks.
These works highlight that even within the SVG domain, different application areas (charts vs. illustrations vs. icons) require tailored approaches to semantic understanding.
Our framework focuses on general-purpose illustrations and icons, where semantic parts correspond to visual objects rather than data encodings.

\subsection{Language models for design tasks}
Recent work has explored leveraging large language models for various design tasks beyond vector graphics.
LayoutPrompter~\cite{lin2023layoutprompter} demonstrated that LLMs can be awakened to perform layout design through carefully crafted prompts that encode spatial relationships and design principles.
PosterO~\cite{hsu2025postero} extended this to content-aware layout generation by structuring layout trees in a way that enables language models to reason about hierarchical spatial arrangements.
These works share our core insight that restructuring visual representations to align with how language models process information is crucial for enabling reliable generation.
Taken together, these insights suggest that LLMs are not inherently incapable of design~\cite{yang2025structeval}, and rather, their potential emerges when design representations are aligned with the natural language structures they are trained to process.

\section{Confidence Bounds for Reliability Estimation}
\phantomsection
\label{anchor-confidence}

In this section, we provide the formal justification for the statistical inference framework described in Section 3.3. 
We first show how to derive the underlying reliabilities from pairwise agreement (\ref{supp:sec-agree}, \ref{supp:sec-spectral}) and then prove that using these reliabilities in a Bayes-weighted vote is provably superior to a standard majority vote (\ref{supp:sec-bounds}).

\subsection{Decomposing agreements and deriving the rank-one structure}
\label{supp:sec-agree}
\phantomsection
\label{anchor-agree}

The goal here is to formalize the relationship between the observable quantity (the agreement patterns $A_{ij}$) and the hidden quantity we care about (the individual reliability of each method, $p_i$).

\begin{lemma}[Agreement] 
Under the symmetric Dawid-Skene model,
$\mA_{ij}=p_ip_j+\frac{(1-p_i)(1-p_j)}{k-1}
=\frac{1}{k}+\frac{k}{k-1}\delta_i\delta_j$ for $i\neq j$, with $\bm\delta_i=p_i-\tfrac{1}{k}$.
\end{lemma}

\begin{proof}
The probability of agreement $\mA_{ij} = \Pr[s_i=s_j]$ is the sum of two mutually exclusive cases. 
One case where both methods are correct ($p_i p_j$), and another case where both methods are incorrect but agree on the same wrong label, which sums to $\frac{(1-p_i)(1-p_j)}{k-1}$. 
The second expression is obtained by substituting $p_i = \frac{1}{k} + \delta_i$ and $1-p_i = \frac{k-1}{k} - \delta_i$ into the first expression and simplifying the resulting algebra. 
\end{proof}

\begin{proposition}[Rank one] 
Let $\mB_{ij}=\mA_{ij}-\tfrac{1}{k}$ ($i\neq j$) and $\mB_{ii}=0$.
Then $\mathbb{E}[\mB]=\tfrac{k}{k-1}\bm\delta\bm\delta^\top$ on off-diagonals (rank one).
\end{proposition}

\begin{proof}
By definition, the centered agreement matrix entry is $\mB_{ij} = \mA_{ij} - \frac{1}{k}$ for $i \ne j$.
Substituting the result from Lemma S.1 for $\mA_{ij}$:
$$ \mB_{ij} = \left( \frac{1}{k} + \frac{k}{k-1}\delta_i\delta_j \right) - \frac{1}{k} = \frac{k}{k-1}\delta_i\delta_j $$
This is the $(i, j)$-th entry of the matrix $\frac{k}{k-1}\bm\delta\bm\delta^\top$. Since $\bm\delta\bm\delta^\top$ is the outer product of the vector $\bm\delta$ with itself, its rank is one, assuming $\bm\delta$ is non-zero.
\end{proof}

\subsection{Reliability estimation via eigenvector analysis}
\label{supp:sec-spectral}
\phantomsection
\label{anchor-spectral}

This is the ``recovery'' part of our proof. 
Now that we have established a link between agreement and skill ($\mB$ and $\bm\delta$), we need to show we can actually solve for $\bm\delta$. 
Theorem S.3 confirms that we can exploit this rank-one structure to reliably estimate the skill vector $\bm\delta$ from our empirical data $\hat{\mB}$ using standard linear algebra techniques, which is finding the top eigenvector.

\begin{theorem}[Quality Guarantee for Estimated Skill]
Let $\hat \mB \in \R^{M\times M}$ be the empirical centered agreement matrix built from $n$ cases and $M$ rendering methods.
If each pairwise agreement is estimated within $\pm\varepsilon$, then with probability $\ge 1-\eta$,
\[
\left\|\hat{\bm\delta} - c\,\bm\delta\right\|_2 \ \le\ C\left(\sqrt{\frac{M}{n}}+\varepsilon\right),
\]
for some confidence bound $\eta$, scale $c>0$, and constant $C$ depending only on $k$.
Furthermore, $c=\sqrt{\lambda_1(k-1)/k}$ where $\lambda_1$ is the top eigenvalue of $\hat B$.%
\footnote{The underlying mathematics, based on the Davis-Kahan theorem, provides a strong upper limit on the error between our calculated skill vector ($\hat{\mathbf{\delta}}$) and the true skill vector ($\mathbf{\delta}$). This error is confirmed to decrease as we process more SVG primitives ($n$).}
\end{theorem}
\begin{proof}
By definition, the centered agreement matrix entry is $\mB_{ij} = \mA_{ij} - \frac{1}{k}$ for $i \ne j$.
Substituting the result from Lemma S.1 for $\mA_{ij}$:
$$ \mB_{ij} = \left( \frac{1}{k} + \frac{k}{k-1}\delta_i\delta_j \right) - \frac{1}{k} = \frac{k}{k-1}\delta_i\delta_j $$
This means the matrix $\mathbb{E}[\mB]$ is a constant factor $\frac{k}{k-1}$ times the matrix $\bm\delta\bm\delta^\top$. The outer product of any vector with itself, $\bm\delta\bm\delta^\top$, is mathematically a rank-one matrix. 
This rank-one property is critical because it means the vector $\bm\delta$ (our reliability or 'skill' vector) must be proportional to the dominant eigenvector of $\mB$, enabling its recovery in Theorem S.3.
\end{proof}

\subsection{Bayes decision rule and error bounds vs.\ Majority voting}
\label{supp:sec-bounds}
\phantomsection
\label{anchor-bounds}

This section proves that weighting VLM responses by their inferred reliability is \textbf{statistically superior} to simple majority voting. Corollary~\ref{cor:improvement-majority} shows that the Bayes-weighted method achieves a strictly better error exponent whenever VLM reliabilities differ.

\subsubsection{Setup and the log-likelihood ratio}

Fix the true label $y^\star \in \{1, 2, \ldots, k\}$ and any competitor label $y \ne y^\star$. For each method $i$, recall that $p_i$ is the probability of correct classification. Define:
\begin{itemize}
    \item $q_i \triangleq \frac{1-p_i}{k-1}$ (probability of any specific wrong label)
    \item $d_i \triangleq p_i - q_i = \frac{kp_i - 1}{k-1}$ (discrimination parameter)
    \item $w_i = \log\frac{p_i}{q_i} = \log\frac{(k-1)p_i}{1-p_i}$ (Bayes weight, or log-likelihood-ratio)
\end{itemize}

\vspace{+0.1cm}
For each observation $s_i$ (method $i$'s output), define the log-likelihood ratio:
\[
Z_i = w_i \cdot \mathbf{1}[s_i = y^\star] - w_i \cdot \mathbf{1}[s_i = y] = \begin{cases}
+w_i & \text{if } s_i = y^\star \\
-w_i & \text{if } s_i = y \\
0 & \text{otherwise}
\end{cases}
\]

The Bayes decision rule prefers $y^\star$ over $y$ when $\sum_i Z_i > 0$. An error occurs when $\sum_i Z_i \le 0$ despite $y^\star$ being true.

\begin{lemma}[Properties of $Z_i$]
Under the true label $y^\star$: (1) $Z_i \in [-w_i, w_i]$, (2) $\mathbb{E}[Z_i \mid y^\star] = w_i d_i$, and (3) the $Z_i$ are independent across methods.
\end{lemma}

\begin{proof}
Boundedness is immediate from the definition. For the expected value, given $y^\star$, method $i$ outputs $y^\star$ with probability $p_i$ (giving $Z_i = w_i$) and outputs $y$ with probability $q_i$ (giving $Z_i = -w_i$). Thus:
\[
\mathbb{E}[Z_i \mid y^\star] = p_i \cdot w_i + q_i \cdot (-w_i) = w_i(p_i - q_i) = w_i d_i.
\]
Independence follows from the conditional independence assumption in the Dawid-Skene model.
\end{proof}

\subsubsection{Error bound for Bayes decision rule}

\begin{theorem}[Hoeffding bound for Bayes LLR]
\label{thm:hoeffding-bayes}
\[
\Pr\!\left[\sum_{i=1}^m Z_i \le 0 \,\bigg|\, y^\star\right]
\ \le\
\exp\!\left(
-\frac{\left(\sum_{i=1}^m w_i d_i\right)^2}{2\sum_{i=1}^m w_i^2}
\right).
\]
\end{theorem}

\begin{proof}
We apply Hoeffding's inequality for bounded independent random variables. Since $Z_i \in [-w_i, w_i]$ with range $(b_i - a_i) = 2w_i$, Hoeffding's inequality gives:
\[
\Pr\!\left[\sum_{i=1}^m (Z_i - \mathbb{E}[Z_i]) \le -t\right] \le \exp\!\left( -\frac{2t^2}{\sum_{i=1}^m 4w_i^2} \right).
\]
The error event $\sum Z_i \le 0$ can be rewritten as $\sum (Z_i - \mathbb{E}[Z_i]) \le -\sum \mathbb{E}[Z_i]$. Setting $t = \sum_{i=1}^m w_i d_i = \sum_{i=1}^m \mathbb{E}[Z_i]$ and substituting:
\[
\Pr\!\left[\sum_{i=1}^m Z_i \le 0\right] \le \exp\!\left( -\frac{2(\sum_{i=1}^m w_i d_i)^2}{4\sum_{i=1}^m w_i^2} \right) = \exp\!\left( -\frac{(\sum_{i=1}^m w_i d_i)^2}{2\sum_{i=1}^m w_i^2} \right).
\]
\end{proof}

The exponent $\frac{(\sum w_i d_i)^2}{2\sum w_i^2}$ quantifies how fast the error probability decays. Larger exponents mean exponentially smaller error rates.

\subsubsection{Comparison with majority voting and proof of superiority}

Majority voting uses uniform weights $w_i^{\text{MV}} \equiv 1$, yielding error exponent $\frac{(\sum_i d_i)^2}{2m}$. We now show that Bayes weighting is strictly better when reliabilities differ.

\begin{theorem}[Improvement over majority voting]
\label{cor:improvement-majority}
In the small-error regime where $|p_i - \frac{1}{k}| \ll 1$, the approximations $w_i \approx \frac{k^2}{k-1}\delta_i$ and $d_i \approx \frac{k}{k-1}\delta_i$ (where $\delta_i = p_i - \frac{1}{k}$) imply $w_i \approx k d_i$. The Bayes error exponent then satisfies:
\[
\text{Exponent}_{\text{BV}} \approx \frac{1}{2}\sum_{i=1}^m d_i^2 = \frac{m}{2}\left[(\text{Mean}(d))^2 + \text{Var}(d)\right] \ge \text{Exponent}_{\text{MV}} = \frac{(\sum d_i)^2}{2m},
\]
with equality if and only if all $d_i$ are equal. The improvement factor is
\[
\frac{\text{Exponent}_{\text{BV}}}{\text{Exponent}_{\text{MV}}} \approx 1 + \frac{\text{Var}(d)}{(\text{Mean}(d))^2},
\]
quantifying the benefit of exploiting heterogeneity in method reliabilities.
\end{theorem}

\begin{proof}
First, we derive the weight approximation. For $p_i = \frac{1}{k} + \delta_i$ with small $\delta_i$:
\[
w_i = \log\frac{(k-1)p_i}{1-p_i} = \log\frac{(k-1)(\frac{1}{k} + \delta_i)}{\frac{k-1}{k} - \delta_i} = \log\frac{1 + k\delta_i}{1 - \frac{k\delta_i}{k-1}}.
\]
Using Taylor expansions $\log(1+x) \approx x$ and $(1-y)^{-1} \approx 1+y$:
\[
w_i \approx k\delta_i + \frac{k\delta_i}{k-1} = \frac{k^2\delta_i}{k-1}.
\]
Similarly, $d_i = \frac{kp_i-1}{k-1} = \frac{k\delta_i}{k-1}$, so $w_i \approx k d_i$.
Now compute the Bayes exponent using $w_i \approx k d_i$:
\[
\text{Exponent}_{\text{BV}} = \frac{(\sum w_i d_i)^2}{2\sum w_i^2} \approx \frac{(k\sum d_i^2)^2}{2k^2\sum d_i^2} = \frac{\sum d_i^2}{2}.
\]
Using the variance decomposition $\sum d_i^2 = m(\text{Mean}(d))^2 + m\cdot\text{Var}(d)$:
\[
\text{Exponent}_{\text{BV}} \approx \frac{m}{2}\left[(\text{Mean}(d))^2 + \text{Var}(d)\right].
\]
For majority voting with $w_i = 1$:
\[
\text{Exponent}_{\text{MV}} = \frac{(\sum d_i)^2}{2m} = \frac{m^2(\text{Mean}(d))^2}{2m} = \frac{m(\text{Mean}(d))^2}{2}.
\]
The difference is $\text{Exponent}_{\text{BV}} - \text{Exponent}_{\text{MV}} \approx \frac{m \cdot \text{Var}(d)}{2} \ge 0$, with equality only when $\text{Var}(d) = 0$ (all $d_i$ equal). The improvement factor is:
\[
\frac{\text{Exponent}_{\text{BV}}}{\text{Exponent}_{\text{MV}}} = \frac{\frac{m}{2}[(\text{Mean}(d))^2 + \text{Var}(d)]}{\frac{m(\text{Mean}(d))^2}{2}} = 1 + \frac{\text{Var}(d)}{(\text{Mean}(d))^2}. \qedhere
\]
\end{proof}

\begin{remark}
The improvement factor $1 + \frac{\text{Var}(d)}{(\text{Mean}(d))^2}$ shows that Bayes weighting provides the most benefit when method reliabilities are heterogeneous. If all methods have identical reliability, both approaches are equivalent. The more diverse the reliabilities, the greater the advantage of properly weighting methods by their estimated skill.
\end{remark}

\clearpage

\section{GPT-Human Alignment on Video-Text Alignment}
\phantomsection
\label{anchor-gpt}

We assess the alignment between our user study and the GPT-T2V metric in order to validate the reliability of the GPT based evaluation. 
In particular, we compare pairwise preferences and measure how often the metric selects the same animation as human participants. 
We find that GPT's preferences (\ie, cases where GPT assigns a higher score to one animation than the other) agree with user preferences in \textbf{$83.4\%$} of the pairs, which indicates a strong correspondence between the automatic and human judgments. 
In comparison, the CLIP-T2V metric, which operates without any external API services, reaches only $53.4\%$ agreement with the user study responses. 
This substantial gap suggests that GPT-T2V captures human perceptual preferences much more faithfully and therefore provides a more reliable proxy for human evaluation in our setting.

We observe that state-of-the-art LLMs demonstrate a robust ability to interpret simple animations and reason about their motion. 
When guided by clear evaluation criteria, such as those provided in \Cref{fig:supp-gpt-prompt}, these models exhibit a high degree of alignment with human judgments. 
Although GPT-5 is also the model used to generate our animations, its role as an evaluator is fundamentally different. 
In practice, using LLMs as judges is often more straightforward and reliable than using them as generators, as evaluation requires consistency and comparative reasoning rather than creative synthesis.

\begin{figure}[h]
    \centering
    \includegraphics[width=\linewidth]{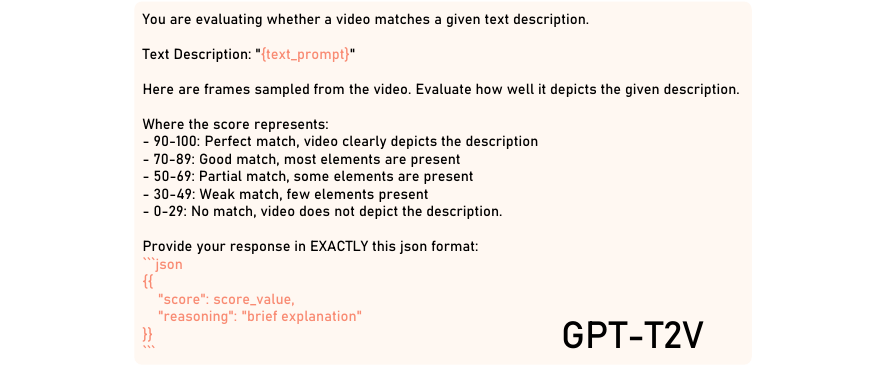}
    \caption{Prompt templated used for GPT-T2V evaluation. }
    \label{fig:supp-gpt-prompt}
\end{figure}

\clearpage

\section{Dataset Composition and Coverage}
\phantomsection
\label{anchor-dataset}

Our test dataset comprises 114 hand-crafted animation instructions across 57 unique SVG files, with each SVG file receiving an average of two distinct animation scenarios. 
These examples were meticulously designed to reflect the diverse animation needs encountered in modern web development. 
As shown in \Cref{tab:subject_theme}, our dataset spans six thematic categories, with particularly strong representation in Nature/Environment (31.6\%) and Objects/Miscellaneous (26.3\%), ensuring broad coverage of visual content types commonly found in web interfaces. 
From tech logos and brand animations to natural phenomena and user interface elements, our dataset encompasses the full spectrum of SVG animation use cases. 
Furthermore, \Cref{tab:interaction_pattern} demonstrates our intentional focus on varied interaction patterns, with Appearance/Reveal animations (28.1\%) and State Transition effects (13.2\%) representing critical components of modern web user experiences. 
The substantial presence of Organic/Natural Movement (12.3\%) and Rotational Movement (8.8\%) patterns reflects our commitment to including both subtle, life-like animations and dynamic, attention-grabbing effects. 
This careful curation ensures that our test dataset not only provides comprehensive coverage but also accurately represents the practical animation requirements of contemporary web applications, from loading indicators and state feedback to decorative enhancements and interactive storytelling.

\begin{table}[h]
\centering
\begin{minipage}{0.48\textwidth}
\centering
\caption{Distribution of Subject Themes in Test Dataset}
\label{tab:subject_theme}
\begin{tabular}{ccc}
\hline
\textbf{Subject Theme} & \textbf{Count} & \textbf{\%} \\
\hline
Nature/Environment & 36 & 31.6 \\
Objects/Miscellaneous & 30 & 26.3 \\
UI/Interface Elements & 18 & 15.8 \\
Tech Logos/Brands & 12 & 10.5 \\
Animals/Characters & 10 & 8.8 \\
Faces/Emojis & 8 & 7.0 \\
\hline
\textbf{Total} & \textbf{114} & \textbf{100.0} \\
\hline
\end{tabular}
\end{minipage}%
\hfill
\begin{minipage}{0.48\textwidth}
\centering
\caption{Distribution of Interaction Patterns in Test Dataset}
\label{tab:interaction_pattern}
\begin{tabular}{lcc}
\hline
\textbf{Interaction Pattern} & \textbf{Count} & \textbf{\%} \\
\hline
Other/Mixed & 43 & 37.7 \\
Appearance/Reveal & 32 & 28.1 \\
State Transition & 15 & 13.2 \\
Organic/Natural Movement & 14 & 12.3 \\
Rotational Movement & 10 & 8.8 \\
\hline
\textbf{Total} & \textbf{114} & \textbf{100.0} \\
\hline
\end{tabular}
\end{minipage}
\end{table}

\clearpage

\section{VLM Prompts for Planning and Animation Generation}
\phantomsection
\label{anchor-prompts}

Our animation pipeline relies on two complementary VLM prompts, one for planning and one for per-class animation generation.

The model is instructed to avoid generic SVG terms and to instead use intuitive, role-based identifiers, and to write a short, human-interpretable description of the intended motion of each part.
This prompt focuses entirely on \emph{semantic intent} and it does not require the model to understand SVG syntax, only to reason visually and symbolically about what should happen.

The second prompt is invoked once for each semantic class produced during restructuring.
It receives three ingredients, the restructured SVG, all previously generated CSS (so it can remain consistent), and the animation plan for that particular class.
Its role is purely \emph{syntactic} and translate one actor's high-level plan into concrete, production-safe CSS.
To avoid conflicts across iterative generations, the prompt enforces a strict ``lanes'' convention in which each motion component (translation, rotation, scale, opacity, blur, etc.) is expressed through typed CSS custom properties rather than direct transform declarations in keyframes.
A single composer rule per class then assembles these properties into the final transform.
This ensures that new animations never overwrite existing ones, allowing independent motions to compose reliably across multiple generation passes.

The two prompts divide responsibilities cleanly, the planner performs semantic reasoning, and the per-class generator performs structured code synthesis.
This separation avoids the common failure modes where a single prompt must juggle visual interpretation, HTML/SVG structure, and CSS constraints simultaneously.
The lanes system further guarantees that iterative code generation remains stable, that different motions do not collide, and that long CSS files can be produced incrementally without exceeding model context limits.

\begin{figure}[h]
    \centering
    \includegraphics[width=0.9\linewidth]{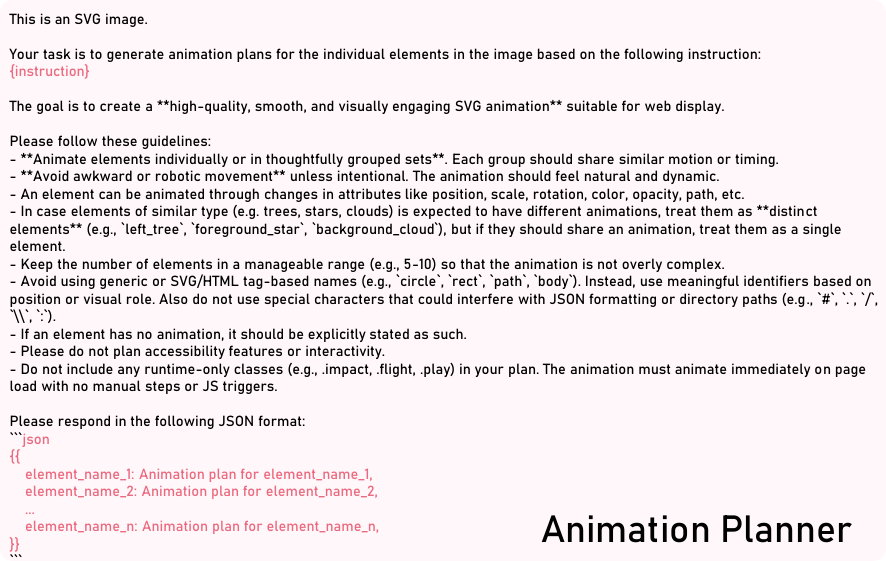}
    \caption{Prompt template used for planning animations. The output is a JSON formatted dictionary of semantic categories and their animation plan. }
    \label{fig:supp-prompt1}
\end{figure}

\begin{figure}[h]
    \centering
    \includegraphics[width=0.9\linewidth]{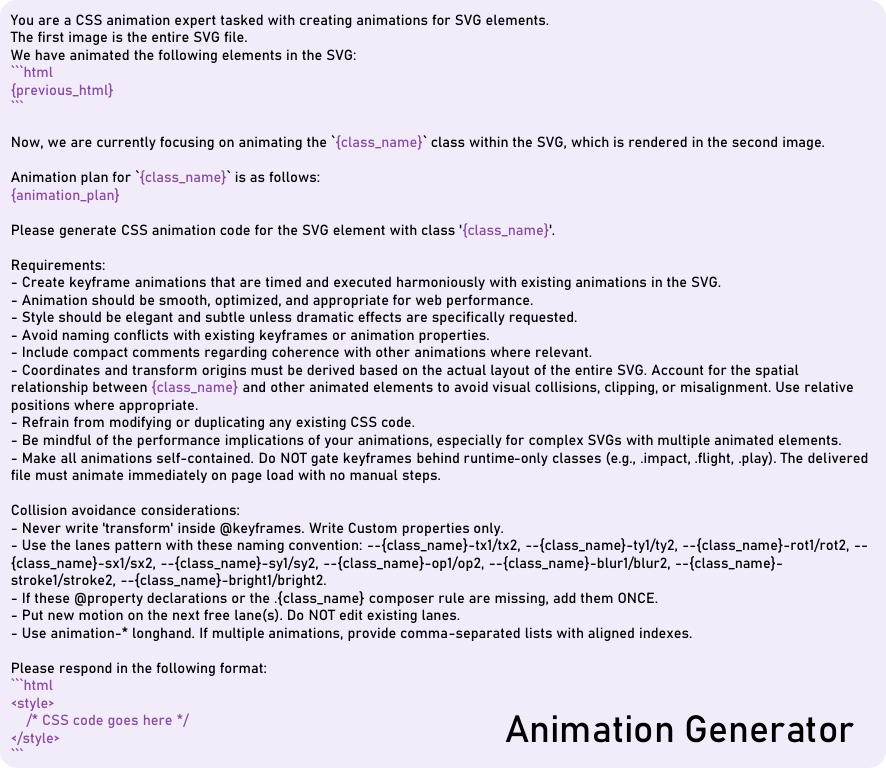}
    \caption{Prompt template used for generating animations. CSS codes are generated in a cascaded manner to bypass generation token length limits.}
    \label{fig:supp-prompt2}
\end{figure}

\clearpage
\section{Restructuring SVG Files with Semantic Labels}
\phantomsection
\label{anchor-restruct}

Once we obtain semantic labels for all primitives, we reorganize the SVG structure by regrouping primitives according to their labels wherever possible. 
This is non-trivial because existing SVG groupings are tightly coupled to the rendering order, and naively introducing new groups can disrupt this order and alter the final appearance.

Nevertheless, restructuring is crucial for enabling meaningful motion, as primitives that belong to the same semantic group can share attributes such as rotation axes, timing, and other animation parameters. 
To safely regroup primitives, we first flatten the SVG structure and ungroup all nested groups, while transferring group properties to the child primities, so that the rendering appears identical to the original. 
Then, we estimate the spatial extent (area) occupied by each primitive and use this to detect conflicting merges. 
Nest, we merge primitives with the same semantic label only when doing so introduces no conflicts with any primitives in between them in the rendering order. 
Finally, we augment each resulting group with metadata, including its bounding box, geometric center, and parent-child relationships, which we later use to drive animation.
We describe the steps in \Cref{alg:regroup} and plan to make all the implementation fully public upon acceptance.

\begin{algorithm}[h]
\caption{Resturcturing SVGs with semantic labels}
\label{alg:regroup}
\begin{algorithmic}[1]
\State \textbf{Inputs} SVG $S$, predicted label $\hat y(x)$ for each primitive $x$
\State \textbf{Output} regrouped SVG $S'$

\State \textbf{Flatten}
\State Traverse $S$ in original paint order and build a list $E=\{(e,\text{idx},\ell,B)\}$
\Statex \hspace{1.2em} $e$ is a cloned primitive with inherited properties baked in
\Statex \hspace{1.2em} $\text{idx}$ is the original paint index
\Statex \hspace{1.2em} $\ell=\hat y(e)$ is appended as the final class token
\Statex \hspace{1.2em} $B$ is a screen-space bounding box

\State \textbf{Regroup by label with a barrier test}
\For{each label $\ell$}
  \State $I_\ell \gets$ indices of $E$ with label $\ell$ in ascending paint order
  \State Greedily form groups $G[\ell]$ over $I_\ell$ using the rule:
  \Statex \hspace{1.2em} a candidate $j$ can join current group $G$ if no element of a different label
  \Statex \hspace{1.2em} whose index lies between $\min(G\cup\{j\})$ and $\max(G\cup\{j\})$
  \Statex \hspace{1.2em} overlaps any member of $G\cup\{j\}$ in screen space
\EndFor

\State \textbf{Compose regrouped SVG}
\State Create $S'$ with original attributes and non-drawables copied verbatim
\State For each group $g$ in order of its earliest index:
\Statex \hspace{1.2em} emit a \texttt{<g>} with class $\ell$-group or $\ell$-group-$k$
\Statex \hspace{1.2em} append members in original relative order
\Statex \hspace{1.2em} write light metadata: bounds, geometric center, paint-order index
\Statex \hspace{1.2em} optionally add parent and children links from the plan
\State \Return $S'$
\end{algorithmic}
\caption{Pseudocode for the SVG file restructuring process using the predicted semantic labels. }
\end{algorithm}

%% file: main_arxiv.bbl
\begin{thebibliography}{46}
\providecommand{\natexlab}[1]{#1}
\providecommand{\url}[1]{\texttt{#1}}
\expandafter\ifx\csname urlstyle\endcsname\relax
  \providecommand{\doi}[1]{doi: #1}\else
  \providecommand{\doi}{doi: \begingroup \urlstyle{rm}\Url}\fi

\bibitem[Blattmann et~al.(2023)Blattmann, Dockhorn, Kulal, Mendelevitch, Kilian, Lorenz, Levi, English, Voleti, Letts, et~al.]{svd}
Andreas Blattmann, Tim Dockhorn, Sumith Kulal, Daniel Mendelevitch, Maciej Kilian, Dominik Lorenz, Yam Levi, Zion English, Vikram Voleti, Adam Letts, et~al.
\newblock Stable video diffusion: Scaling latent video diffusion models to large datasets.
\newblock \emph{arXiv preprint arXiv:2311.15127}, 2023.

\bibitem[Cai et~al.(2023)Cai, Huang, Li, Ojha, Wang, and Lee]{cai2023leveraging}
Mu Cai, Zeyi Huang, Yuheng Li, Utkarsh Ojha, Haohan Wang, and Yong~Jae Lee.
\newblock Leveraging large language models for scalable vector graphics-driven image understanding.
\newblock \emph{arXiv preprint arXiv:2306.06094}, 2023.

\bibitem[Carlier et~al.(2020)Carlier, Danelljan, Alahi, and Timofte]{deepsvg}
Alexandre Carlier, Martin Danelljan, Alexandre Alahi, and Radu Timofte.
\newblock Deepsvg: A hierarchical generative network for vector graphics animation.
\newblock \emph{NeurIPS}, 2020.

\bibitem[Chaturvedi et~al.(2021)Chaturvedi, Luk{\'a}{\v{c}}, and Chaudhuri]{regroup}
Sumit Chaturvedi, Michal Luk{\'a}{\v{c}}, and Siddhartha Chaudhuri.
\newblock Regroup: Recursive neural networks for hierarchical grouping of vector graphic primitives.
\newblock \emph{arXiv preprint arXiv:2111.11759}, 2021.

\bibitem[Chen et~al.(2023)Chen, Lee, Wang, Chang, and Liu]{chen2023mystique}
Chen Chen, Bongshin Lee, Yunhai Wang, Yunjeong Chang, and Zhicheng Liu.
\newblock Mystique: Deconstructing svg charts for layout reuse.
\newblock \emph{IEEE TVCG}, 2023.

\bibitem[Chen et~al.(2024)Chen, Bako, Yu, Hooker, Joyal, Wang, Kim, Wu, Ding, Sandeep, et~al.]{chen2024visanatomy}
Chen Chen, Hannah~K Bako, Peihong Yu, John Hooker, Jeffrey Joyal, Simon~C Wang, Samuel Kim, Jessica Wu, Aoxue Ding, Lara Sandeep, et~al.
\newblock Visanatomy: An svg chart corpus with fine-grained semantic labels.
\newblock \emph{arXiv preprint arXiv:2410.12268}, 2024.

\bibitem[Chen(2021)]{chen2021evaluating}
Mark Chen.
\newblock Evaluating large language models trained on code.
\newblock \emph{arXiv preprint arXiv:2107.03374}, 2021.

\bibitem[Davies and Bouldin(2009)]{dbi}
David~L Davies and Donald~W Bouldin.
\newblock A cluster separation measure.
\newblock \emph{IEEE TPAMI}, 2009.

\bibitem[Dawid and Skene(1979)]{dawid1979maximum}
Alexander~Philip Dawid and Allan~M Skene.
\newblock Maximum likelihood estimation of observer error-rates using the em algorithm.
\newblock \emph{Journal of the Royal Statistical Society: Series C (Applied Statistics)}, 1979.

\bibitem[Grahn(2011)]{theanimatepackage}
Alexander Grahn.
\newblock The animate package, 2011.

\bibitem[Group(2020)]{vtracer}
Vision~Cortex Group.
\newblock Vtracer, 2020.

\bibitem[Hsu and Peng(2025)]{hsu2025postero}
HsiaoYuan Hsu and Yuxin Peng.
\newblock Postero: Structuring layout trees to enable language models in generalized content-aware layout generation.
\newblock In \emph{CVPR}, 2025.

\bibitem[Li et~al.(2022)Li, Choi, Chung, Kushman, Schrittwieser, Leblond, Eccles, Keeling, Gimeno, Dal~Lago, et~al.]{alphacode}
Yujia Li, David Choi, Junyoung Chung, Nate Kushman, Julian Schrittwieser, R{\'e}mi Leblond, Tom Eccles, James Keeling, Felix Gimeno, Agustin Dal~Lago, et~al.
\newblock Competition-level code generation with alphacode.
\newblock \emph{Science}, 378\penalty0 (6624):\penalty0 1092--1097, 2022.

\bibitem[Lin et~al.(2023)Lin, Guo, Sun, Yang, Lou, and Zhang]{lin2023layoutprompter}
Jiawei Lin, Jiaqi Guo, Shizhao Sun, Zijiang Yang, Jian-Guang Lou, and Dongmei Zhang.
\newblock Layoutprompter: Awaken the design ability of large language models.
\newblock \emph{NeurIPS}, 2023.

\bibitem[Lin et~al.(2025)Lin, Zheng, Ran, Zhu, Mao, Li, Torr, and Wang]{lin2025vcode}
Kevin~Qinghong Lin, Yuhao Zheng, Hangyu Ran, Dantong Zhu, Dongxing Mao, Linjie Li, Philip Torr, and Alex~Jinpeng Wang.
\newblock Vcode: a multimodal coding benchmark with svg as symbolic visual representation.
\newblock \emph{arXiv preprint arXiv:2511.02778}, 2025.

\bibitem[Liu et~al.(2023)Liu, Li, Wu, and Lee]{llava}
Haotian Liu, Chunyuan Li, Qingyang Wu, and Yong~Jae Lee.
\newblock Visual instruction tuning.
\newblock \emph{NeurIPS}, 36, 2023.

\bibitem[Liu et~al.(2025)Liu, Meng, Ouyang, Yu, Zhao, Cohen-Or, and Qu]{dynamictypo}
Zichen Liu, Yihao Meng, Hao Ouyang, Yue Yu, Bolin Zhao, Daniel Cohen-Or, and Huamin Qu.
\newblock Dynamic typography: Bringing text to life via video diffusion prior.
\newblock In \emph{ICCV}, 2025.

\bibitem[Nijkamp et~al.(2023)Nijkamp, Pang, Hayashi, Tu, Wang, Zhou, Savarese, and Xiong]{codegen}
Erik Nijkamp, Bo Pang, Hiroaki Hayashi, Lifu Tu, Huan Wang, Yingbo Zhou, Silvio Savarese, and Caiming Xiong.
\newblock Codegen: An open large language model for code with multi-turn program synthesis.
\newblock In \emph{ICLR}, 2023.

\bibitem[Nishina and Matsui(2025)]{svgeditbench}
Kunato Nishina and Yusuke Matsui.
\newblock Svgeditbench v2: A benchmark for instruction-based svg editing.
\newblock \emph{arXiv preprint arXiv:2502.19453}, 2025.

\bibitem[OpenAI(2025{\natexlab{a}})]{gpt5}
OpenAI.
\newblock Gpt-5, 2025{\natexlab{a}}.

\bibitem[OpenAI(2025{\natexlab{b}})]{sora2}
OpenAI.
\newblock Sora2 system card, 2025{\natexlab{b}}.

\bibitem[Poole et~al.(2023)Poole, Jain, Barron, and Mildenhall]{sds}
Ben Poole, Ajay Jain, Jonathan~T Barron, and Ben Mildenhall.
\newblock Dreamfusion: Text-to-3d using 2d diffusion.
\newblock In \emph{ICLR}, 2023.

\bibitem[Qiu et~al.(2025)Qiu, Liu, Feng, Liu, Xiao, Collins, Tenenbaum, Weller, Black, and Sch{\"o}lkopf]{qiu2024can}
Zeju Qiu, Weiyang Liu, Haiwen Feng, Zhen Liu, Tim~Z Xiao, Katherine~M Collins, Joshua~B Tenenbaum, Adrian Weller, Michael~J Black, and Bernhard Sch{\"o}lkopf.
\newblock Can large language models understand symbolic graphics programs?
\newblock \emph{ICLR}, 2025.

\bibitem[Radford et~al.(2021)Radford, Kim, Hallacy, Ramesh, Goh, Agarwal, Sastry, Askell, Mishkin, Clark, et~al.]{clip}
Alec Radford, Jong~Wook Kim, Chris Hallacy, Aditya Ramesh, Gabriel Goh, Sandhini Agarwal, Girish Sastry, Amanda Askell, Pamela Mishkin, Jack Clark, et~al.
\newblock Learning transferable visual models from natural language supervision.
\newblock In \emph{International Conference on Machine Learning}. PmLR, 2021.

\bibitem[Rodriguez et~al.(2025)Rodriguez, Agarwal, Laradji, Rodriguez, Vazquez, Pal, and Pedersoli]{starvector}
Juan~A Rodriguez, Shubham Agarwal, Issam~H Laradji, Pau Rodriguez, David Vazquez, Christopher Pal, and Marco Pedersoli.
\newblock Starvector: Generating scalable vector graphics code from images.
\newblock In \emph{CVPR}, 2025.

\bibitem[Rombach et~al.(2022)Rombach, Blattmann, Lorenz, Esser, and Ommer]{ldm}
Robin Rombach, Andreas Blattmann, Dominik Lorenz, Patrick Esser, and Bj{\"o}rn Ommer.
\newblock High-resolution image synthesis with latent diffusion models.
\newblock In \emph{CVPR}, 2022.

\bibitem[Sim{\'e}oni et~al.(2025)Sim{\'e}oni, Vo, Seitzer, Baldassarre, Oquab, Jose, Khalidov, Szafraniec, Yi, Ramamonjisoa, Massa, Haziza, Wehrstedt, Wang, Darcet, Moutakanni, Sentana, Roberts, Vedaldi, Tolan, Brandt, Couprie, Mairal, J{\'e}gou, Labatut, and Bojanowski]{simeoni2025dinov3}
Oriane Sim{\'e}oni, Huy~V. Vo, Maximilian Seitzer, Federico Baldassarre, Maxime Oquab, Cijo Jose, Vasil Khalidov, Marc Szafraniec, Seungeun Yi, Micha{\"e}l Ramamonjisoa, Francisco Massa, Daniel Haziza, Luca Wehrstedt, Jianyuan Wang, Timoth{\'e}e Darcet, Th{\'e}o Moutakanni, Leonel Sentana, Claire Roberts, Andrea Vedaldi, Jamie Tolan, John Brandt, Camille Couprie, Julien Mairal, Herv{\'e} J{\'e}gou, Patrick Labatut, and Piotr Bojanowski.
\newblock {DINOv3}, 2025.

\bibitem[Wan et~al.(2025)Wan, Wang, Ai, Wen, Mao, Xie, Chen, Yu, Zhao, Yang, Zeng, Wang, Zhang, Zhou, Wang, Chen, Zhu, Zhao, Yan, Huang, Feng, Zhang, Li, Wu, Chu, Feng, Zhang, Sun, Fang, Wang, Gui, Weng, Shen, Lin, Wang, Wang, Zhou, Wang, Shen, Yu, Shi, Huang, Xu, Kou, Lv, Li, Liu, Wang, Zhang, Huang, Li, Wu, Liu, Pan, Zheng, Hong, Shi, Feng, Jiang, Han, Wu, and Liu]{wan2025}
Team Wan, Ang Wang, Baole Ai, Bin Wen, Chaojie Mao, Chen-Wei Xie, Di Chen, Feiwu Yu, Haiming Zhao, Jianxiao Yang, Jianyuan Zeng, Jiayu Wang, Jingfeng Zhang, Jingren Zhou, Jinkai Wang, Jixuan Chen, Kai Zhu, Kang Zhao, Keyu Yan, Lianghua Huang, Mengyang Feng, Ningyi Zhang, Pandeng Li, Pingyu Wu, Ruihang Chu, Ruili Feng, Shiwei Zhang, Siyang Sun, Tao Fang, Tianxing Wang, Tianyi Gui, Tingyu Weng, Tong Shen, Wei Lin, Wei Wang, Wei Wang, Wenmeng Zhou, Wente Wang, Wenting Shen, Wenyuan Yu, Xianzhong Shi, Xiaoming Huang, Xin Xu, Yan Kou, Yangyu Lv, Yifei Li, Yijing Liu, Yiming Wang, Yingya Zhang, Yitong Huang, Yong Li, You Wu, Yu Liu, Yulin Pan, Yun Zheng, Yuntao Hong, Yupeng Shi, Yutong Feng, Zeyinzi Jiang, Zhen Han, Zhi-Fan Wu, and Ziyu Liu.
\newblock Wan: Open and advanced large-scale video generative models.
\newblock \emph{arXiv preprint arXiv:2503.20314}, 2025.

\bibitem[Wang et~al.(2025{\natexlab{a}})Wang, Yin, Wei, Zeng, Gu, Ye, Gao, Wang, Zhang, Li, et~al.]{wang2025internsvg}
Haomin Wang, Jinhui Yin, Qi Wei, Wenguang Zeng, Lixin Gu, Shenglong Ye, Zhangwei Gao, Yaohui Wang, Yanting Zhang, Yuanqi Li, et~al.
\newblock Internsvg: Towards unified svg tasks with multimodal large language models.
\newblock \emph{arXiv preprint arXiv:2510.11341}, 2025{\natexlab{a}}.

\bibitem[Wang et~al.(2024)Wang, He, Li, Li, Yu, Ma, Li, Chen, Chen, Wang, et~al.]{viclip}
Yi Wang, Yinan He, Yizhuo Li, Kunchang Li, Jiashuo Yu, Xin Ma, Xinhao Li, Guo Chen, Xinyuan Chen, Yaohui Wang, et~al.
\newblock Internvid: A large-scale video-text dataset for multimodal understanding and generation.
\newblock In \emph{ICLR}, 2024.

\bibitem[Wang et~al.(2025{\natexlab{b}})Wang, Huang, Sun, Gong, Cohen-Or, and Lu]{wang2025layered}
Zhenyu Wang, Jianxi Huang, Zhida Sun, Yuanhao Gong, Daniel Cohen-Or, and Min Lu.
\newblock Layered image vectorization via semantic simplification.
\newblock In \emph{CVPR}, 2025{\natexlab{b}}.

\bibitem[Wei et~al.(2022)Wei, Wang, Schuurmans, Bosma, Xia, Chi, Le, Zhou, et~al.]{wei2022chain}
Jason Wei, Xuezhi Wang, Dale Schuurmans, Maarten Bosma, Fei Xia, Ed Chi, Quoc~V Le, Denny Zhou, et~al.
\newblock Chain-of-thought prompting elicits reasoning in large language models.
\newblock \emph{NeurIPS}, 2022.

\bibitem[Wu et~al.(2023)Wu, Zhang, Liao, Chen, Hou, Wang, Sun, Yan, and Lin]{wu2023dover}
Haoning Wu, Erli Zhang, Liang Liao, Chaofeng Chen, Jingwen~Hou Hou, Annan Wang, Wenxiu~Sun Sun, Qiong Yan, and Weisi Lin.
\newblock Exploring video quality assessment on user generated contents from aesthetic and technical perspectives.
\newblock In \emph{ICCV}, 2023.

\bibitem[Wu et~al.(2025{\natexlab{a}})Wu, Su, and Liao]{layerpeeler}
Ronghuan Wu, Wanchao Su, and Jing Liao.
\newblock Layerpeeler: Autoregressive peeling for layer-wise image vectorization.
\newblock \emph{SIGGRAPH Asia}, 2025{\natexlab{a}}.

\bibitem[Wu et~al.(2025{\natexlab{b}})Wu, Su, Ma, and Liao]{aniclipart}
Ronghuan Wu, Wanchao Su, Kede Ma, and Jing Liao.
\newblock Aniclipart: Clipart animation with text-to-video priors.
\newblock \emph{IJCV}, 2025{\natexlab{b}}.

\bibitem[Xing et~al.(2025)Xing, Hu, Liang, Zhang, Xu, and Yu]{xing2025empowering}
Ximing Xing, Juncheng Hu, Guotao Liang, Jing Zhang, Dong Xu, and Qian Yu.
\newblock Empowering llms to understand and generate complex vector graphics.
\newblock In \emph{CVPR}, 2025.

\bibitem[Yang et~al.(2025{\natexlab{a}})Yang, Li, Yang, Zhang, Hui, Zheng, Yu, Gao, Huang, Lv, et~al.]{qwen3}
An Yang, Anfeng Li, Baosong Yang, Beichen Zhang, Binyuan Hui, Bo Zheng, Bowen Yu, Chang Gao, Chengen Huang, Chenxu Lv, et~al.
\newblock Qwen3 technical report.
\newblock \emph{arXiv preprint arXiv:2505.09388}, 2025{\natexlab{a}}.

\bibitem[Yang et~al.(2025{\natexlab{b}})Yang, Jiang, He, Siu, Zhang, Liao, Li, Zeng, Jia, Wang, et~al.]{yang2025structeval}
Jialin Yang, Dongfu Jiang, Lipeng He, Sherman Siu, Yuxuan Zhang, Disen Liao, Zhuofeng Li, Huaye Zeng, Yiming Jia, Haozhe Wang, et~al.
\newblock Structeval: Benchmarking llms' capabilities to generate structural outputs.
\newblock \emph{arXiv preprint arXiv:2505.20139}, 2025{\natexlab{b}}.

\bibitem[Yang et~al.(2025{\natexlab{c}})Yang, Jimenez, Zhang, Lieret, Yang, Wu, Press, Muennighoff, Synnaeve, Narasimhan, et~al.]{swe}
John Yang, Carlos~E Jimenez, Alex~L Zhang, Kilian Lieret, Joyce Yang, Xindi Wu, Ori Press, Niklas Muennighoff, Gabriel Synnaeve, Karthik~R Narasimhan, et~al.
\newblock Swe-bench multimodal: Do ai systems generalize to visual software domains?
\newblock In \emph{ICLR}, 2025{\natexlab{c}}.

\bibitem[Yang et~al.(2025{\natexlab{d}})Yang, Cheng, Chen, Zeng, Yin, Zhang, Wang, Yu, Ma, and Jiang]{omnisvg}
Yiying Yang, Wei Cheng, Sijin Chen, Xianfang Zeng, Fukun Yin, Jiaxu Zhang, Liao Wang, Gang Yu, Xingjun Ma, and Yu-Gang Jiang.
\newblock Omnisvg: A unified scalable vector graphics generation model.
\newblock In \emph{NeurIPS}, 2025{\natexlab{d}}.

\bibitem[Yuan et~al.(2025)Yuan, Zhang, Deng, Zhang, Zhu, Jia, Zhou, and Zhang]{yuan2025rdtf}
Zhiqiang Yuan, Ting Zhang, Ying Deng, Jiapei Zhang, Yeshuang Zhu, Zexi Jia, Jie Zhou, and Jinchao Zhang.
\newblock Rdtf: Resource-efficient dual-mask training framework for multi-frame animated sticker generation.
\newblock \emph{arXiv preprint arXiv:2503.17735}, 2025.

\bibitem[Zhang et~al.(2024)Zhang, Zhao, and Liao]{zhang2024text}
Peiying Zhang, Nanxuan Zhao, and Jing Liao.
\newblock Text-to-vector generation with neural path representation.
\newblock \emph{ACM TOG}, pages 1--13, 2024.

\bibitem[Zhang et~al.(2025)Zhang, Zhao, and Liao]{stylecustomization}
Peiying Zhang, Nanxuan Zhao, and Jing Liao.
\newblock Style customization of text-to-vector generation with image diffusion priors.
\newblock In \emph{SIGGRAPH}, pages 1--11, 2025.

\bibitem[Zhang et~al.(2023)Zhang, Liu, Zhang, Cheng, and Wang]{beyondpixels}
Tong Zhang, Haoyang Liu, Peiyan Zhang, Yuxuan Cheng, and Haohan Wang.
\newblock Beyond pixels: Exploring human-readable svg generation for simple images with vision language models.
\newblock \emph{arXiv preprint arXiv:2311.15543}, 2023.

\bibitem[Zheng et~al.(2023)Zheng, Chiang, Sheng, Zhuang, Wu, Zhuang, Lin, Li, Li, Xing, et~al.]{zheng2023judging}
Lianmin Zheng, Wei-Lin Chiang, Ying Sheng, Siyuan Zhuang, Zhanghao Wu, Yonghao Zhuang, Zi Lin, Zhuohan Li, Dacheng Li, Eric Xing, et~al.
\newblock Judging llm-as-a-judge with mt-bench and chatbot arena.
\newblock \emph{NeurIPS}, 2023.

\bibitem[Zou et~al.(2024)Zou, Cai, and Lee]{zouvgbench}
Bocheng Zou, Mu Cai, and Jianrui Zhang Yong~Jae Lee.
\newblock Vgbench: Evaluating large language models on vector graphics understanding and generation.
\newblock \emph{EMNLP}, 2024.

\end{thebibliography}
